\documentclass{article}
\usepackage{ijcai20}

\usepackage{times}
\usepackage{soul}
\usepackage{url}
\usepackage[hidelinks]{hyperref}
\usepackage[utf8]{inputenc}
\usepackage[small]{caption}
\usepackage{graphicx}
\usepackage{amsmath}
\usepackage{amssymb}
\usepackage{amsthm}
\usepackage{booktabs}
\usepackage{algorithm}
\usepackage{algorithmic}
\usepackage{subfig}
\newcommand{\Lim}[1]{\raisebox{0.5ex}{\scalebox{0.8}{$\displaystyle \lim_{#1}\;$}}}

\newtheorem{theorem}{Theorem}

\title{Predictive Modeling through Hyper-Bayesian Optimization}

\author{
First Author$^1$
\and
Second Author$^2$\and
Third Author$^{2,3}$\And
Fourth Author$^4$
}

\begin{document}

\maketitle

\begin{abstract}
	Model selection is an integral problem of model based optimization techniques such as Bayesian optimization (BO). Current approaches often treat model selection as an estimation problem, to be periodically updated with observations coming from the optimization iterations. In this paper, we propose an alternative way to achieve both efficiently. Specifically, we propose a novel way of integrating model selection and BO for the single goal of reaching the function optima faster. The algorithm moves back and forth between BO in the model space and BO in the function space, where the goodness of the recommended model is captured by a score function and fed back, capturing how well the model helped convergence in the function space. The score function is derived in such a way that it neutralizes the effect of the moving nature of the BO in the function space, thus keeping the model selection problem stationary. This back and forth leads to quick convergence for both model selection and BO in the function space. In addition to improved sample efficiency, the framework outputs information about the black-box function. Convergence is proved, and experimental results show significant improvement compared to standard BO.
\end{abstract}
\section{Introduction}
In cognitive science, “Predictive process modeling” is used as a plausible biological model of the brain \cite{Clark2013,Clark2015a}. The  mind is depicted as a multi-layer prediction machine performing top-down predictions of the world that are met with bottom-up streams of sensory data.  The top layers are continuously predicting what is about to be sensed. The model is refined or altered to fit the incoming signals, and this “dance” continues until “equilibrium” is reached. 

\begin{figure}[!tb]
	\begin{center}
		\includegraphics[scale=0.25]{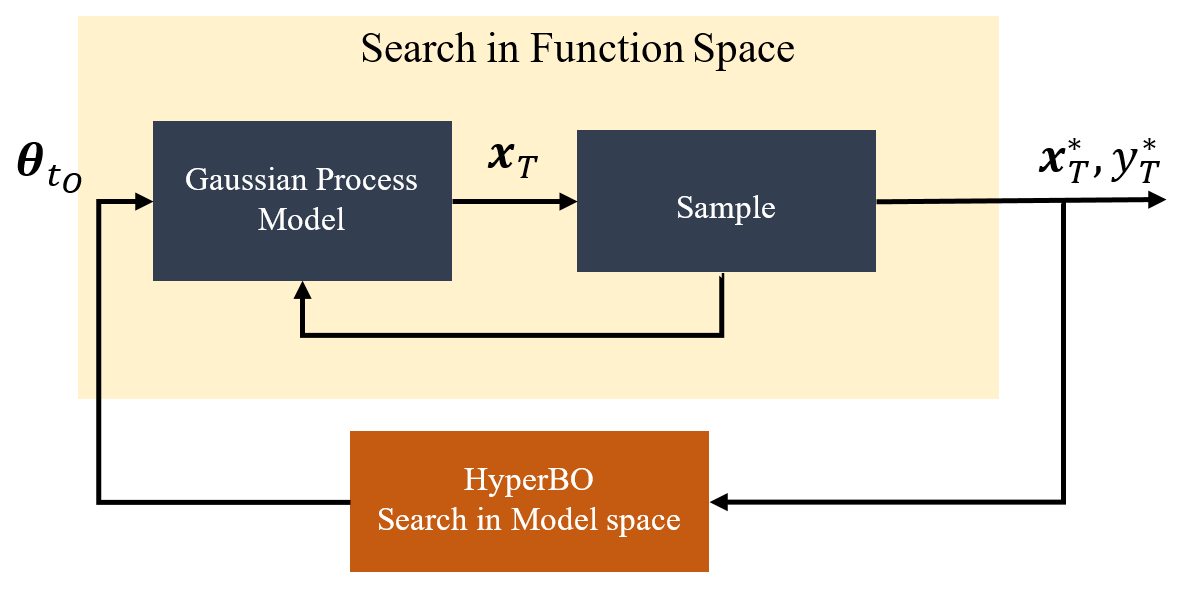}
	\end{center}
	\caption{Predictive modeling framework with HyperBO. Optima search conducted by BO in function space (yellow), its hyperparameters $\boldsymbol{\theta}_{t_{O}}$ chosen by BO in model space (orange).  $\boldsymbol{x}_{t}$ is recommendation and $\boldsymbol{x}^{*}_{t}$,$y^{*}_{t}$ is best recommendation. $t_O$ is the BO iteration in the model space and $T$ is the BO iteration in the function space.}
	\label{fig:HyerBO_framework}
\end{figure}

We draw inspiration from this idea and present the Predictive Modeling Framework, integrating model selection and Bayesian Optimization (BO) to accelerate finding the function optimum (Figure \ref{fig:HyerBO_framework}). Our framework uses a top level model generator (HyperBO) which is modeled through BO, an efficient method for global optimization of expensive and noisy black-box functions \cite{DBLP:journals/corr/abs-1012-2599,Snoek2012,Freitas2015}. It searches in the model space, recommending  hyperparameters ($\boldsymbol{\theta}$) for the underlying BO modeling the function space. The algorithm moves between BO in the model space and BO in the function space, where the recommended model's goodness is fed back using a score function that captures how well the model helped convergence in the function space. The score function is a key feature as it is able to offset for the moving nature of BO in the function space. This back and forth leads to quick convergence for both model selection and function optimisation, keeping the stationarity of HyperBO. The framework improves sample efficiency and outputs information about the black-box. Existing convergence analysis of BO algorithms assume an optimal model. We prove convergence w.r.t. a changing model, being optimised alongside the optima search. Our framework is evaluated on five real world regression problems. The model selection involves finding the monotonicity information that best describes the function (Diabetes and Boston Housing), and the kernel parameters that best model the function (Concrete Compression, Fish toxicity, and Power Plant). Our contributions are as follows:

\begin{itemize}
	\item Development of Predictive Modeling framework for BO incorporating automatic model searching (HyperBO);
	\item Development of a stationary scoring function to review performance of the model search; and
	\item Convergence analysis of our framework.
\end{itemize}

\subsection{Related Work}
In the space of automating model design, \cite{Duvenaud2013} proposes a generative kernel grammar with the ability of defining the space of kernel structures which are built compositionally from a small number of base kernels. \cite{malkomes2016bayesian} proposes the BOMS method, an automatic framework for exploring the potential space of models using Bayesian Optimization. Building on this work to implement model searching within the optimization search of a black-box function \cite{Malkomes:2018:ABO:3327345.3327498} presents ABOMS. This work utilised Bayesian optimization in the model search space to determine the best model hyperparameters, whilst at the same time optimizing the black box function. The fit of the model is measured by the normalized marginal likelihood, allowing comparison of models across iterations. However, this is an indirect measure of the goodness towards the optimization.

\subsection{Problem definition}

We assume a black-box function $f:  \Re^{D}\rightarrow \Re$ with an observation model of $y=f(\textbf{x})+\epsilon$ where $\epsilon \sim N(0,\sigma_{n}^{2})$, $\textbf{x} \in \Re^{D}$ and $y \in \Re$. Our goal is to discover  $\textbf{x} \in \chi \subset \Re^{D}$, where $\chi$ is a compact subspace of $\Re^{D}$ such that:

\begin{equation}\label{eq:BO_optimiser}
\textbf{x}^{*} = \operatorname*{argmax}_{x\in \chi}  f(\textbf{x})
\end{equation}

\subsection{Background}
\subsubsection{Bayesian Optimization}
Bayesian optimization is a sample efficient method for the optimization of black box functions. The optimization occurs by constructing a Gaussian Process (GP) model of the function. GP is used as a prior over the space of smooth functions and can be fully described by a mean and covariance function \cite{Rasmussen:2005:GPM:1162254}. Assuming a mean of zero, the GP is then defined by its covariance function $f \sim GP(0,k(\textbf{x},\textbf{x'}))$. Given some observations $D = \{\textbf{x}_i,y_i\}_{i=1}^{t}$, the posterior is also a GP for which the predictive mean and covariance can be computed via,

\begin{equation} \label{eq:mu}
\mu_{t}(\textbf{x}) = \textbf{k}_{t}(\textbf{x})^T(K_{t} + \sigma_{n}^2\mathbf{I})^{-1}\textbf{y}_{t}
\end{equation}
\begin{equation} \label{eq:sigma}
k_{t}(\textbf{x},\textbf{x}') = k(\textbf{x},\textbf{x}') - \textbf{k}_t(\textbf{x})^T(K_t + \sigma_{n}^2\mathbf{I})^{-1}\textbf{k}_t(\textbf{x}')
\end{equation}
with variance $\sigma_{t}^2(\textbf{x}) = k_{t}(\textbf{x},\textbf{x})$, $\textbf{k}_t(\textbf{x}) = [k(\textbf{x}_i,\textbf{x})]_{i=1}^t$ and, $K_t = [k(\textbf{x}_t,\textbf{x}_{t^{'}})]_{t,t^{'}}$ is the kernel Gram matrix.

The squared exponential kernel of the form shown in (\ref{eq:sekernel}) is a popular choice of kernel and will be used in this paper.

\begin{equation} \label{eq:sekernel}
k(\textbf{x}_i,\textbf{x}_j) = \sigma_{f}^2 \text{exp}(-\frac{1}{2} \sum_{d=1}^{D} \frac{(x_{d,i} - x_{d,j})^2}{l_{d}^2})
\end{equation}
where $\sigma_{f}^2$ is the signal variance, and $l_{d}$ is a constant length scale for the $d$-th dimension.

The mean and variance from the GP is used to construct an acquisition function $a(\boldsymbol{x})$ which is then optimized to select the next best sample point. 

\begin{equation}\label{eq:BO_acquisition}
\textbf{x}_{t} = \operatorname*{argmax}_{x\in \chi}  a_{t}(\textbf{x})
\end{equation}

There are multiple acquisition functions used in Bayesian optimization problems including probability of improvement, GP-UCB, Expected Improvement (EI) and Thompson Sampling \cite{DBLP:journals/corr/abs-1012-2599,Basu2017}.

\section{Proposed Framework}

\subsection{HyperBO Framework}
Our overarching objective is to search for the optima of the black-box function. Whilst conducting the Bayesian optimization search for the function optima, the model search for this BO is  guided by another BO in the model space, we call it HyperBO. After every $K$ iterations of the inner BO in the function space, a new model, described by a vector of its hyperparameters ($\boldsymbol{\theta}$), is applied. After $K$ iterations with this new model, its performance in progressing the search to the optima is assessed using a scoring function ($g(\boldsymbol{\theta})$). This score is then fed back to the GP within HyperBO ($GP^{\theta}$). The iterations of the BO in the model space are denoted by $t_O$. We use Thompson Sampling in order to select the next best model hyperparameter values for the HyperBO. The choice of Thompson sampling as an acquisition function is guided by the need to analyse the convergence of our method, where we use the result of \cite{Basu2017}. However, we use GP-UCB as an acquisition function for the BO in the function space, as it provides a stricter convergence guarantee. Over time, the HyperBO will converge to the ideal model hyperparameters, accelerating the convergence of the BO in the function space. We note that by using HyperBO for model selection, we use a more active approach of model selection using prediction, rather than the usual passive way of using estimation. Algorithms \ref{alg:AMMSF} - \ref{alg:HyperBO} describe this Predictive modeling framework.

Mathematically, our optimization problem is defined as:
\begin{align*}
\boldsymbol{\theta}^{*} = & \operatorname*{max}_{\boldsymbol{\theta}\in \boldsymbol{\Theta}} g(\boldsymbol{\theta})
\end{align*}

where $g(\boldsymbol{\theta}) \triangleq$ score function
\begin{align*}
A_{p}^{T} : & p, \boldsymbol{\theta}, T \rightarrow \boldsymbol{x} \text{ (inner optimization)}\\
p \triangleq & \operatorname*{max}_{\boldsymbol{x}\in \boldsymbol{\chi}} f(\boldsymbol{x})
\end{align*}
Where $T$ is the iteration of the BO in the function space, $g(\boldsymbol{\theta})$ is the scoring function, $A_{p}^{T}$ is the BO algorithm for problem $p$, at $T$'th iteration, and $p$ is the optimization problem in the function space. To maintain the stationarity of the HyperBO optimization problem, i.e. maximizing the scoring function $g(\boldsymbol{\theta})$, we should make the scoring function independent of the iteration ($T$) of the BO in the function space. In the next we show how to use the convergence results of Bayesian optimization to derive such a scoring function.

\begin{algorithm}[t]
	\caption{Predictive Modeling Framework}
	\label{alg:AMMSF}
	\textbf{Input}: black-box function $f$, initial data $\boldsymbol{D}_0$ made of (\textbf{x},\textbf{y})\\
	\textbf{Parameter}: number of initial models $m$, number of iteration per model $K$, number of total iterations $R$\\
	\textbf{Output}: $\boldsymbol{x}^{*}$ and $\boldsymbol{\theta}^{*}$
	\begin{algorithmic}[1] %[1] enables line numbers
		\STATE $M_0 \gets \emptyset$
		\STATE Construct initial model set
		\FOR{$t_{O} = 1, 2, ..., m$}
		\STATE $y^{+} \leftarrow max(\boldsymbol{y})$
		\STATE Generate random hyperparameter vector $\boldsymbol{\theta}_{t_{O}}$
		\STATE Generate score $S_{t_{O}}$ for $\boldsymbol{\theta}_{t_{O}}$ and output $\boldsymbol{D}_{t_{O}}$ from Alg 2: Model Score
		\STATE Update model data $\boldsymbol{M}_{t_{O}}=\boldsymbol{M}_{t_{O}-1} \cup (\boldsymbol{\theta}_{t_{O}},S_{t_{O}})$
		\ENDFOR
		\STATE Begin Bayesian Optimization of black-box function
		\WHILE{$t_{O} \leq R$}
		\STATE $y^{+} \leftarrow \text{max}(\boldsymbol{y})$
		\STATE Generate $\boldsymbol{\theta}_{t_{O}}$ from Alg 3: HyperBO
		\STATE Generate score $S_{t_{O}}$ for $\boldsymbol{\theta}_{t_{O}}$ and output $\boldsymbol{D}_{t_{O}}$ from Alg 2: Model Score
		\STATE Update model data $\boldsymbol{M}_{t_{O}}=\boldsymbol{M}_{t_{O}-1} \cup (\boldsymbol{\theta}_{t_{O}},S_{t_{O}})$
		\ENDWHILE
	\end{algorithmic}
\end{algorithm}
\begin{algorithm}
	\caption{Model Score}
	\label{alg:ScoringModel}
	\textbf{Input}: black-box function $f$, data $\boldsymbol{D}_{t_{O}-1}$ hyperparameter vector $\boldsymbol{\theta}_{t_{O}}$, current optima $y^{+}$\\
	\textbf{Parameter}: number of iteration per model $K$\\
	\textbf{Output}:model score $S_{t_{O}}$, data $\boldsymbol{D}_{t_{O}}$
	\begin{algorithmic}[1]
		\STATE $\boldsymbol{D}_{T} = \boldsymbol{D}_{t_{O}-1}$
		\FOR{$i = 1, 2, ..., K$}
		\STATE $T = t_{O}+i$
		\STATE Fit $GP_{t_{O}}$ with $\boldsymbol{\theta}_{t_{O}}$ and $\boldsymbol{D}_{T-1}$
		\STATE Compute $\mu_{T}$ and $\sigma_{T}$ according to (\ref{eq:mu}) and (\ref{eq:sigma})
		\STATE Construct and maximise acquisition function $a_{t}$
		\STATE Sample function $y_{T} = f(\boldsymbol{x}_{T})$ and update data $\boldsymbol{D}_{T} = \boldsymbol{D}_{{T}-1} \cup (\boldsymbol{x}_{T},y_{T})$
		\ENDFOR
		\STATE $f^{+}(A_p(\boldsymbol{\theta})) \leftarrow \text{max}(\boldsymbol{y})$
		\STATE Score performance $S_{t_{O}}$ of $\boldsymbol{\theta}_{t_{O}}$ with (\ref{eq:score2})
		\STATE $\boldsymbol{D}_{t_O} = \boldsymbol{D}_{T}$
	\end{algorithmic}
\end{algorithm}
\begin{algorithm}
	\caption{HyperBO}
	\label{alg:HyperBO}
	\textbf{Input}: model data $\boldsymbol{M}_{t_{O}}$, model score $S$\\
	\textbf{Output}: hyperparameter $\boldsymbol{\theta}_{t_{O}}$
	\begin{algorithmic}[1]
		\STATE $s^{+} \leftarrow \text{max}(S)$
		\STATE Fit $GP_{t_{O}}^{\theta}$ with $\boldsymbol{M}_{t_{O}}$
		\STATE Compute $\mu_{t_{O}}$ and $\sigma_{t_{O}}$ according to (\ref{eq:mu}) and (\ref{eq:sigma})
		\STATE Select next $\boldsymbol{\theta}_{t_{O}}$ through Thompson Sampling
	\end{algorithmic}
\end{algorithm}

\subsection{Scoring Function}

As mentioned before, a key innovation in our scoring metric is the ability to offset for the moving nature of the BO search, thereby effectively allowing it to function as a black-box that can be analysed independently. In order to maintain this required stationarity, we describe this reduction in progress towards the optima in later iterations via the regret bound of GP-UCB as detailed in \cite{Srinivas:2010:GPO:3104322.3104451}, utilizing it as a means of describing the average regret. 

Regret Bound $R_{T}$ is of $O(\sqrt{T\gamma_{T}\text{log}(\mid D \mid)})$ where $\gamma_{T}$ is of $O((\text{log} T)^{d+1})$. $T$ is the iteration number.

\begin{align*}
\frac{R_{T}}{T} = & O\left(\sqrt{\frac{\gamma_{T} \text{log}(\mid D \mid)}{T}} \right)\\
\frac{R_{T}}{T} = & O\left(\sqrt{\frac{(\text{log} T)^{d+1} \text{log}(\mid D \mid)}{T}} \right)\\
\frac{R_{T}}{T} \sim & C\sqrt{\frac{\text{log} T^{d+1}}{T}}
\end{align*}

For the score we can ignore $C$ as it is a constant. We use the shape of $\frac{R_{T}}{T}$ to normalize the gain in performance during any interval of the BO iterations in the function space, making the gain independent of the iteration ($T$). This is what we use as our scoring function.

\begin{equation}\label{eq:score1}
g(\boldsymbol{\theta}) = \frac{f^{+}(A_{p}(\boldsymbol{\theta}))-y^{+}}{\sqrt{\frac{\text{log } T^{d+1}}{T}}}
\end{equation}
where $y^{+}$ is the best observation before BO in the function space with the new hyperparameter, $f^{+}(A_{p}(\boldsymbol{\theta}))$ is the best observation after BO with the new hyperparameter.

In practise, it was observed that adding a regularization term to (\ref{eq:score1}) can improve the performance further. This term differs slightly depending on the hyperparameter type being tuned. Some examples will be discussed in the later sections.

\subsection{Convergence Analysis}

The Predictive Modeling framework with HyperBO can be proven to converge under the Theorem below.

\begin{theorem}
	Let $\delta \in (0,1)$ and $\eta \in (0,1)$. Assuming the kernel functions used in both the BO in the function space and the HyperBO, $k(.,.)$ provides a high probability guarantee on the sample paths of GP derivative to be L'-Lipschitz continuous, and the function $f(A^{t}_{p}(\boldsymbol{\theta}))$ is L-Lipschitz continuous. There exists a $t_O = T_{S} \leq T_{O}$ beyond which $\left\|\boldsymbol{\theta}^{*}-\boldsymbol{\theta}_{t_{O}}\right\| < \epsilon$ is satisfied with probability $1-\delta$. Furthermore, the average cumulative regret of the Predictive Modeling framework will converge to $\Lim{T_{I} \rightarrow \infty} R_{T}/T = \epsilon L$.
\end{theorem}

\begin{proof}
If BO with Thompson Sampling is used \cite{Basu2017}, then we know that
\begin{align*}
Prob(\left\|\boldsymbol{\theta}^{*}-\boldsymbol{\theta}_{t}\right\|>\epsilon) \leq C^{0}_{\epsilon}\text{exp}(-C^{1}_{\epsilon}t)
\end{align*}
where $C^{0}_{\epsilon}$, $C^{1}_{\epsilon}$ are $\epsilon$ dependent constants.

This implies we can set any arbitrary $\epsilon \ll 1$ and an arbitrary low probability $\delta \ll 1$. Then there exists a $t_{O} = T_{S}$ beyond which $\left\|\boldsymbol{\theta}^{*}-\boldsymbol{\theta}_{t_{O}}\right\| < \epsilon$ happens with high probability ($1 - \delta$).
\begin{align*}
\delta = & C^{0}_{\epsilon}\text{exp}(-C^{1}_{\epsilon}T_{S})\\
\implies T_{S} = & C^{1}_{\epsilon}\text{log}(\frac{C^{0}_{\epsilon}}{\delta})
\end{align*}

Using regret as defined in \cite{Srinivas:2010:GPO:3104322.3104451}, and recalling that $A_{p}^{T}(\boldsymbol{\theta}_{t})=\boldsymbol{x}_{T}$, we can write the cumulative regret as:

\begin{align*}
R_{T} = & \sum_{t=1}^{T_{O}}\sum_{t'=1}^{K}|f(\boldsymbol{x}^{*})-f(A_{p}^{T}(\boldsymbol{\theta}_{t}))|
\end{align*}
where $T=(t-1)\times K+t'$ is the actual iterations that the BO in the function space has gone through. Next, we break down $R_{T}$ by introducing $f(A_{p}^{T}(\boldsymbol{\theta}^{*}))$, where $\boldsymbol{\theta}^{*}=\text{argmax }g(\boldsymbol{\theta})$.
\begin{align*}
R_T= & \sum_{t=1}^{T_{O}}\sum_{t'=1}^{K} |f(\boldsymbol{x}^{*})-f(A_{p}^{T}(\boldsymbol{\theta}^{*})) +f(A_{p}^{T}(\boldsymbol{\theta}^{*}))\\
&-f(A_{p}^{T}(\boldsymbol{\theta}_t))|\\ 
\leq & \underbrace{\sum_{t=1}^{T}|f(\boldsymbol{x}^{*})-f(A_{p}^{T}(\boldsymbol{\theta}^{*}))|}_{O(\sqrt{T\text{log}T})}\\
&+\sum_{t=1}^{T_{O}}\sum_{t'=1}^{K} |f(A_{p}^{T}(\boldsymbol{\theta}^{*}))-f(A_{p}^{T}(\boldsymbol{\theta}_t))|\\ 
\leq & O(\sqrt{T\text{log}T})+\\
&T_{0} \operatorname*{max}_{t= [1,T_{0}]}(K \operatorname*{max}_{t'= [1,K]}\underbrace{|f(A_{p}^{T}(\boldsymbol{\theta}^{*}))-f(A_{p}^{T}(\boldsymbol{\theta}_t))|}_{L\left\|\boldsymbol{\theta}^{*}-\boldsymbol{\theta}_{t}\right\|})
\end{align*}
Because the $GP$ predictive posterior is smooth w.r.t $\boldsymbol{\theta}$ it makes the acquisition function to be smooth as well w.r.t. $\boldsymbol{\theta}$. So we can assume that $f(A_{p}^{T}(\boldsymbol{\theta}))$ is L-Lipschitz. Hence,
\begin{align*}
R_T\leq & O(\sqrt{T\text{log}T})+\underbrace{T_{0} K}_{T} L\underbrace{\left\|\boldsymbol{\theta}^{*}-\boldsymbol{\theta}_{t}\right\|}_{\epsilon} \\ 
\leq & O(\sqrt{T\text{log}T})+T \epsilon L \\ 
& \text{Taking the limit }\Lim{T \rightarrow \infty} \frac{R_{T}}{T} = \epsilon L
\end{align*}
\end{proof}
Although the regret does not vanish in our case, it can be made arbitrary small by setting $\epsilon$ very small. We must also note that the existing convergence analysis assumes that the best model is being used throughout the BO, ignoring the effect of estimation of the model on the convergence. In contrast, we provide the convergence guarantee of our whole approach, including the model selection part, thus making it more useful to look at.

\subsection{Application 1 - Length Scale Tuning}

It is common for the length scale of the model GP to be tuned periodically with observations during the optimization process. We instead separate the length scale tuning outside of the BO and into the HyperBO framework in order to achieve faster convergence. The length scales for each dimensions, expressed within the $\boldsymbol{\theta}$ vector, are assumed to be anisotropic.

\subsubsection{Length Scale Regularization}
A regularization is applied to the score function to weight the score towards larger length scale values to prevent over sensitivity to a given input dimension. The overall score function used is given by
\begin{equation}\label{eq:length_score}
g(\boldsymbol{\theta}) = \frac{f^{+}(A_{p}(\boldsymbol{\theta}))-y^{+}}{\sqrt{\frac{\text{log} T^{d+1}}{T}}}\left(1-\lambda\left\|\boldsymbol{\theta}\right\|\right)
\end{equation}
where, $\lambda$ is the regularization weight.

\subsection{Application 2 - Monotonicity Tuning}

If a function is monotonic to certain variables, access to monotonicity information can greatly improve the fit of a GP and increase the sample efficiency and convergence during the BO search \cite{Wang2018,Golchi2015,Li2018}. However it is often the case that this information is not available prior to experimentation. As such, we apply our framework for monotonicity discovery.

\subsubsection{Gaussian Process with monotonicity}

\cite{riihimaki2010gaussian} proposed a method of incorporating monotonicity information into a Gaussian process model. The monotonicity is added at some finite locations of the search space through an observation model, via virtual derivative observations. In this method a parameter $\nu$ is used to control the strictness of the monotonicity information. Further details on the theory and implementation of this method can be found at \cite{riihimaki2010gaussian}.

\subsubsection{Searching monotonic models}

Unlike \cite{riihimaki2010gaussian} which requires monotonicity information be incorporated into the GP prior, we work off the assumption that monotonicity information is unknown. Our proposed framework uses HyperBO to search the space of GP models, and in the case of monotonicity, discover the best monotonic GP model to describe the latent function. 

\cite{riihimaki2010gaussian} specified the strictness parameter $\nu$ as a constant. In our work we utilise this parameter instead as a way of discovering monotonicity in a model, by tuning its value to reflect the strength of monotonicity in a given direction and dimension.
To do this we assume independent $\nu$ values in each dimension thereby describing the strictness parameter as a vector $\boldsymbol{\nu}$. In addition we describe monotonicity in both increasing and decreasing directions for each dimension. The reason for having monotonicity directions be described independently is to have the elements of $\boldsymbol{\nu}$ range between $10^{-6}$ - $10^{0}$ where lower values apply a strong monotonic trend. In order to have an effective search, we conduct the search in the log space, thereby searching for values between -6 to 0, where a lower value (-6) invokes a strong monotonicity. We represent  $\boldsymbol{\nu}$ as the $\boldsymbol{\theta}$ in our algorithm. 

As an example, for a 2 dimensional function, we construct a vector $\boldsymbol{\theta}$ for the monotonicity directions [-1 1 -2 2] with $\boldsymbol{\theta} = [\theta_{1}^{-}, \theta_{1}^{+}, \theta_{2}^{-}, \theta_{2}^{+}]$, where $\theta_{d}^{-}$ is the strictness parameter for decreasing monotonicity and $\theta_{d}^{+}$ is for increasing monotonicity in the d-th dimension. 

\begin{figure}[t]
	\centering
	\subfloat[Concrete]{{\includegraphics[height = 3cm, width=3.5cm]{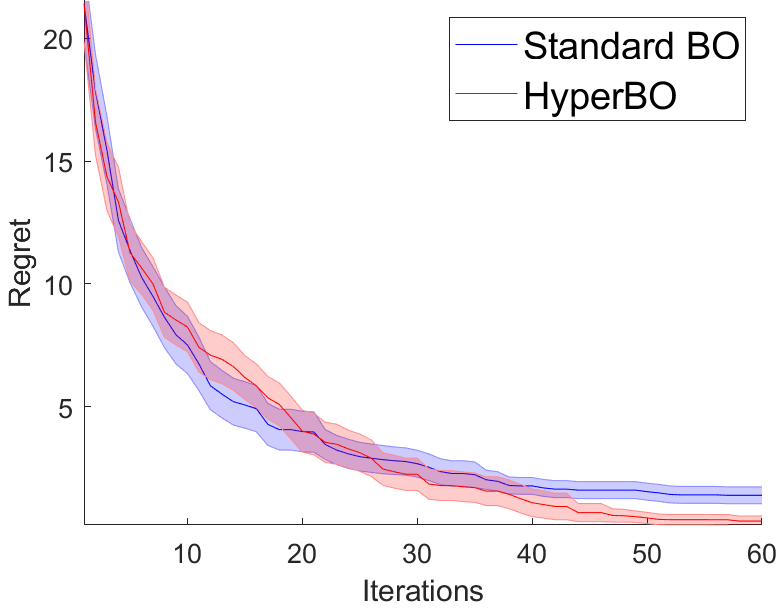}}}%
	\qquad
	\subfloat[Concrete]{{\includegraphics[height = 3cm,width=3.5cm]{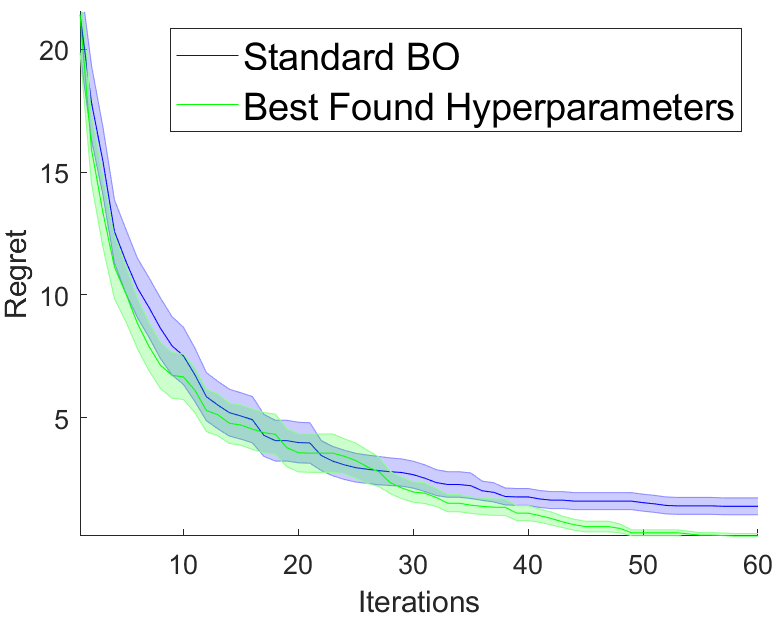}}}%
	\qquad
	\subfloat[Power Plant]{{\includegraphics[height = 3cm,width=3.5cm]{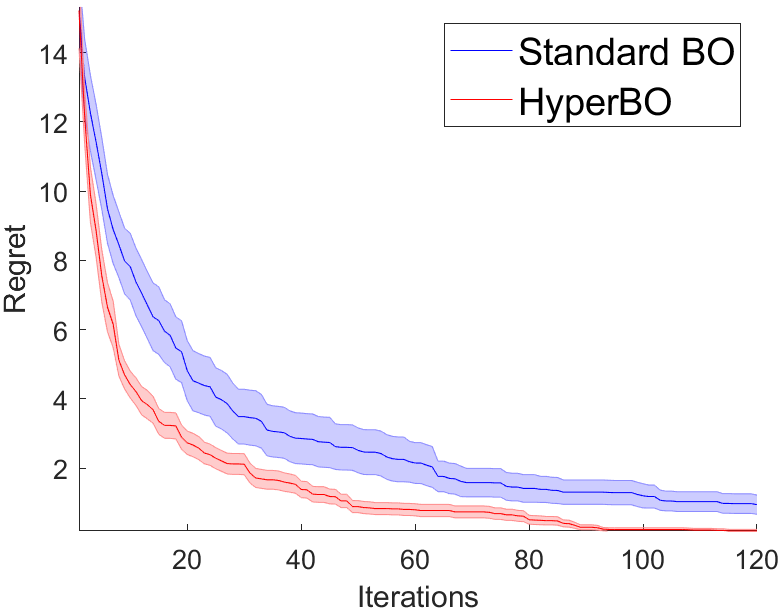}}}%
	\qquad
	\subfloat[Power Plant]{{\includegraphics[height = 3cm,width=3.5cm]{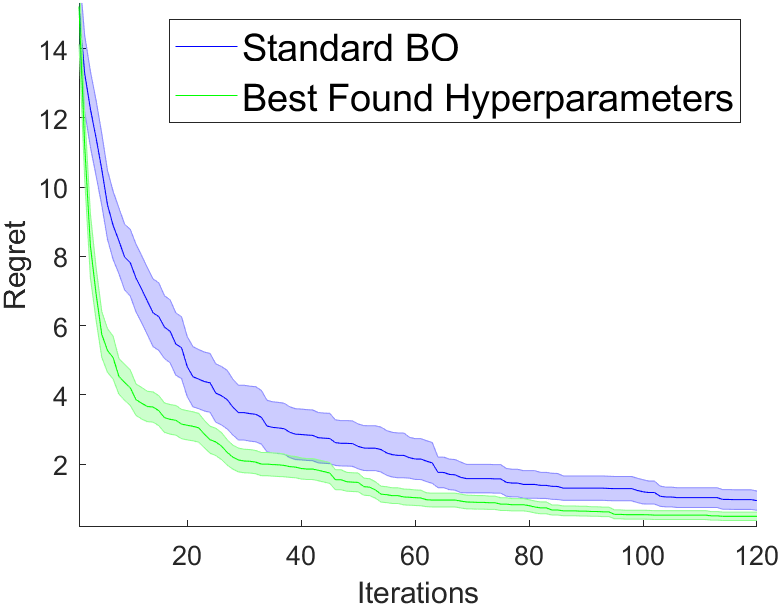}}}%
	\qquad
	\subfloat[Fish Toxicity]{{\includegraphics[height = 3cm,width=3.5cm]{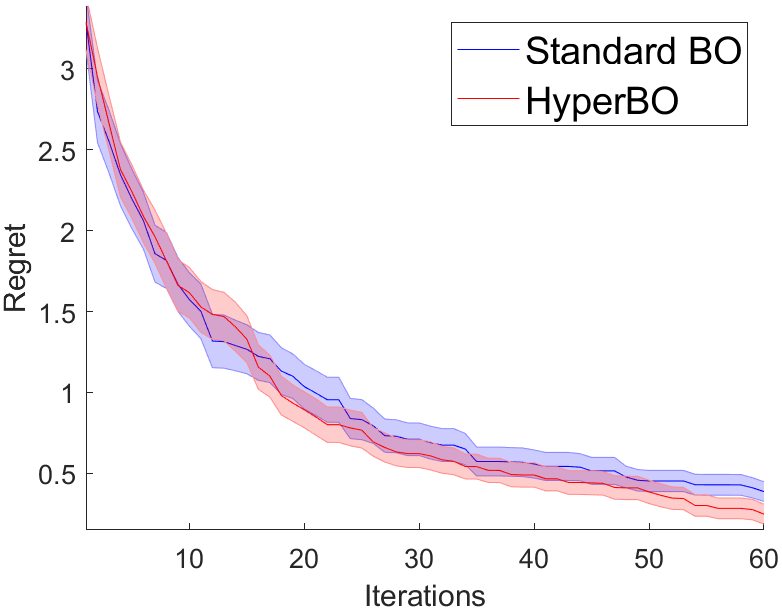}}}%
	\qquad
	\subfloat[Fish Toxicity]{{\includegraphics[height = 3cm,width=3.5cm]{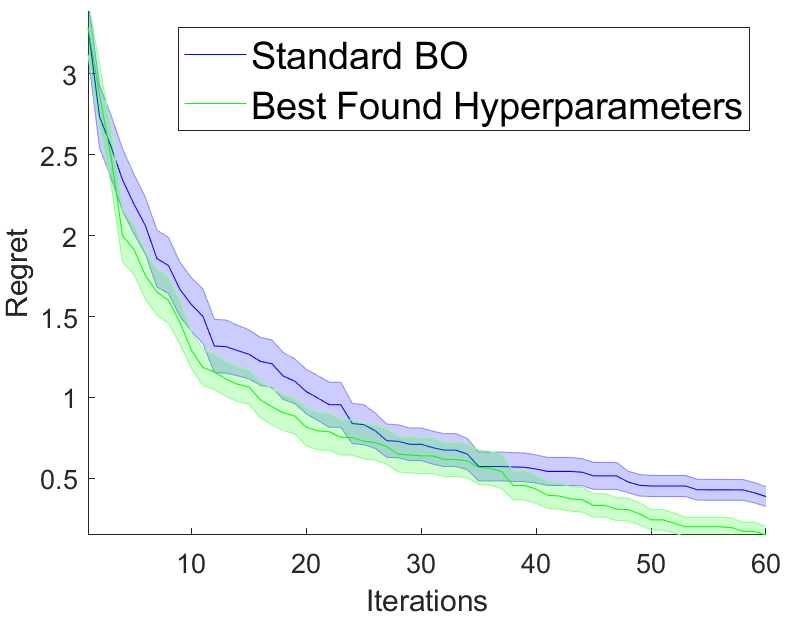}}}%
	\caption{Regret vs Iteration results of Length Scale tuning experiments. Case 1: HyperBO vs BO performance: a,c,e) Demonstrates Predictive modeling framework is better able to discover appropriate length scales to reach the function optima.  b,d,f): Case 2: GP with best monotonicity information outperforms Standard BO.}%
	\label{fig:LengthScale_experiments}%
\end{figure}

\subsubsection{Monotonicity Regularization}
Regularization of the form shown in (\ref{eq:score2}) is applied for scoring monotonic models. A low value of $\boldsymbol{\theta}$ indicates the presence of monotonicity, while a high value indicates a lack of monotonicity. As such a weighting is applied to the score in order to prefer large values of $\boldsymbol{\theta}$ and thereby prevent the presence of monotonicity being over emphasised. In our experiments, the values of $\theta \in [-6, 0]$. The weighting is designed to favour values closer to 0. Below is the form of the score function used with $\lambda$ being the regularization weight.
\begin{equation}\label{eq:score2}
g(\boldsymbol{\theta}) = \frac{f^{+}(A_{p}(\boldsymbol{\theta}))-y^{+}}{\sqrt{\frac{\text{log } T^{d+1}}{T}}}\left(1+\lambda\left\|\boldsymbol{\theta}\right\|\right)
\end{equation}

\section{Experimental Results}

Experimentation was conducted for discovering length scales and monotonicity information of synthetic functions and real world datasets using the Predictive modeling framework. In the experiments presented the framework was used to discover either monotonicity or the length scale values, but the framework can be extended to tune both, if required. Average of 50 trials are reported with error bars indicating standard error. The input ranges of each experiment were normalized. 

The results are shown in two cases: 

\textit{Case 1}: HyperBO vs Standard BO; 

\textit{Case 2}: Optimization performance of Best discovered hyperparameter values - re-running experiment with best scoring $\boldsymbol{\theta}$ from HyperBO run. This is used to test the validity of the hyperparameters found by HyperBO. If prior information is available about the true hyperparameter values (gold standard), these are compared against as well.

Code for all Length Scale experiments and synthetic monotonicity experiment can be found at https://bit.ly/2tElQBX. 
\subsection{Length Scale Experiments}
The Predictive modeling framework was applied to 3 real world datasets for tuning length scales, which included Concrete Compressive Strength, Combined Cycle Power Plant and QSAR fish toxicity, whilst searching for the maxima of the dataset. The length scale ranges were discretized between 0.1 and 0.6, with a step size of 0.05.  The performance of HyperBO in discovering the appropriate length scales for the models is displayed in Figure \ref{fig:LengthScale_experiments}. For these  experiments, the value of $\lambda$ in the scoring function (\ref{eq:length_score}) was set to $\frac{1}{ 0.6 \times d}$ in order to normalize the regularization term. In all experiments, tuning of length scale with the framework improves the performance, as shown in Figure \ref{fig:LengthScale_experiments}.

\subsubsection{Concrete Compressive Strength}
This dataset comprised of 8 variables from 1030 experiments to predict the compressive strength of concrete. The task of BO is to find the maximum compressive strength using the existing data points as a discrete search space.

\subsubsection{Combined Cycle Power Plant}
This dataset comprised of 4 variables from 9568 data points. The input variables, which include temperature, ambient pressure, relative humidity and exhaust vacuum predict  the hourly electrical energy output of a power plant. Here the task is to determine the maximum hourly energy output.

\subsubsection{QSAR fish toxicity}
This dataset is comprised of 6 molecular descriptors of 908 chemicals predicting toxicity towards the fish species known as Pimephales promelas (fathead minnow). The task here is to determine the maximum toxicity towards the fish.

\begin{figure}[t]
	\centering
	\subfloat[Goldstein-Price]{{\includegraphics[height = 3cm, width=3.5cm]{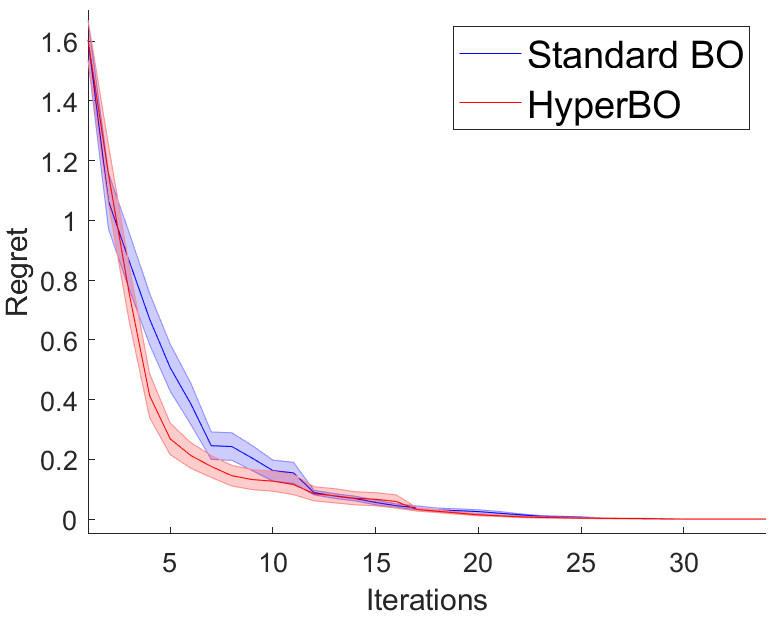}}}%
	\qquad
	\subfloat[Goldstein-Price]{{\includegraphics[height = 3cm,width=3.5cm]{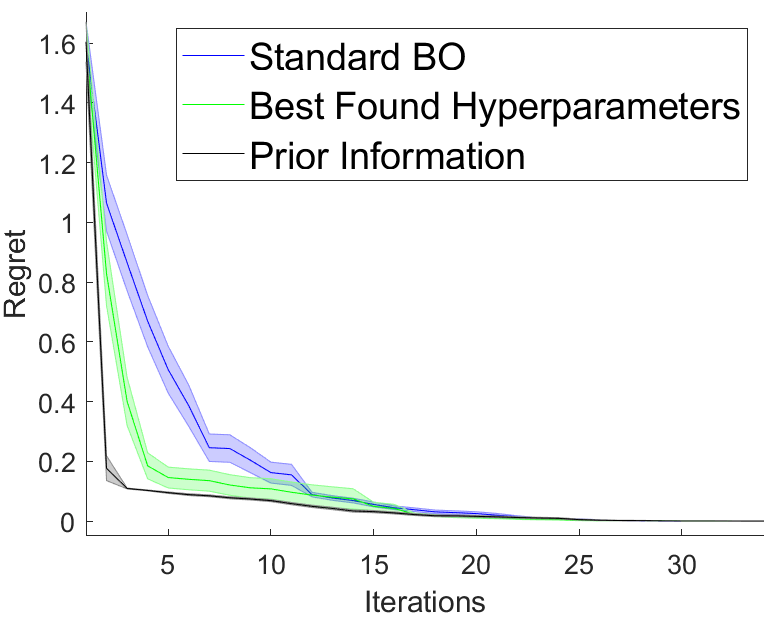}}}%
	\qquad
	\subfloat[Diabetes]{{\includegraphics[height = 3cm,width=3.5cm]{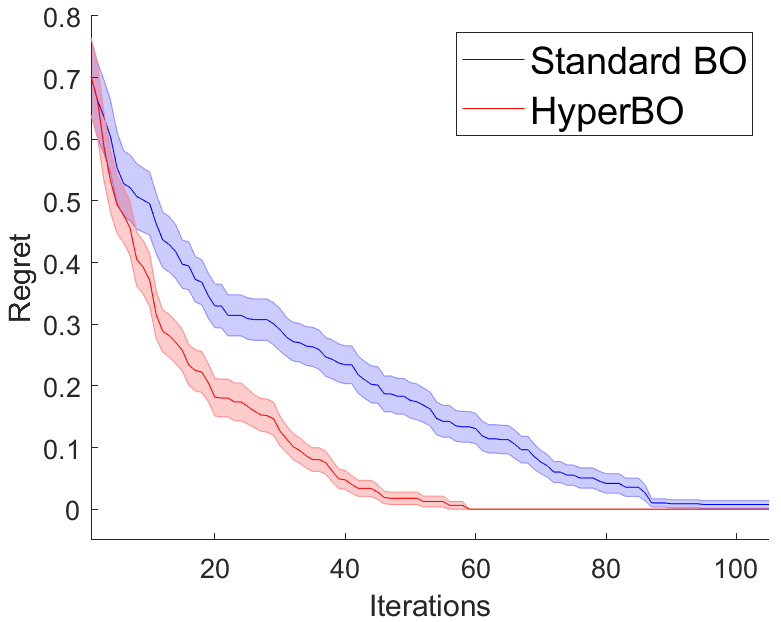}}}%
	\qquad
	\subfloat[Diabetes]{{\includegraphics[height = 3cm,width=3.5cm]{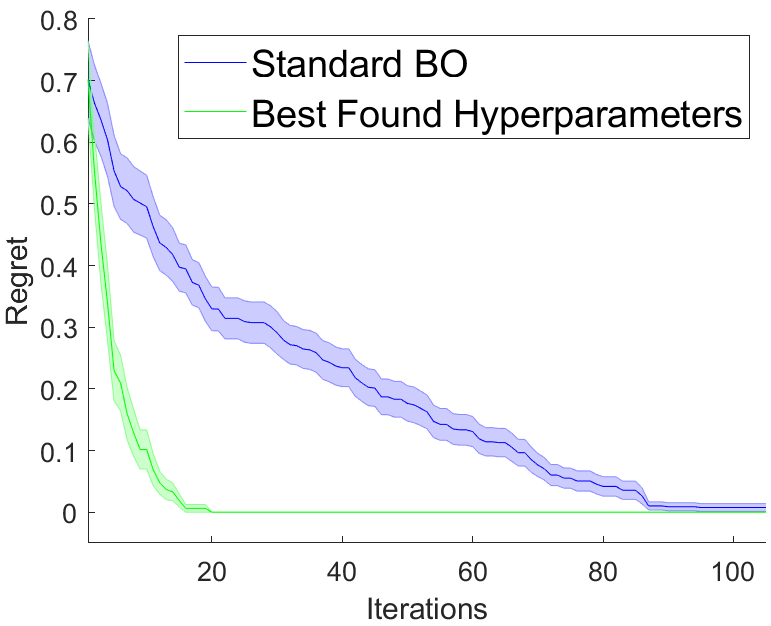}}}%
	\qquad
	\subfloat[Boston]{{\includegraphics[height = 3cm,width=3.5cm]{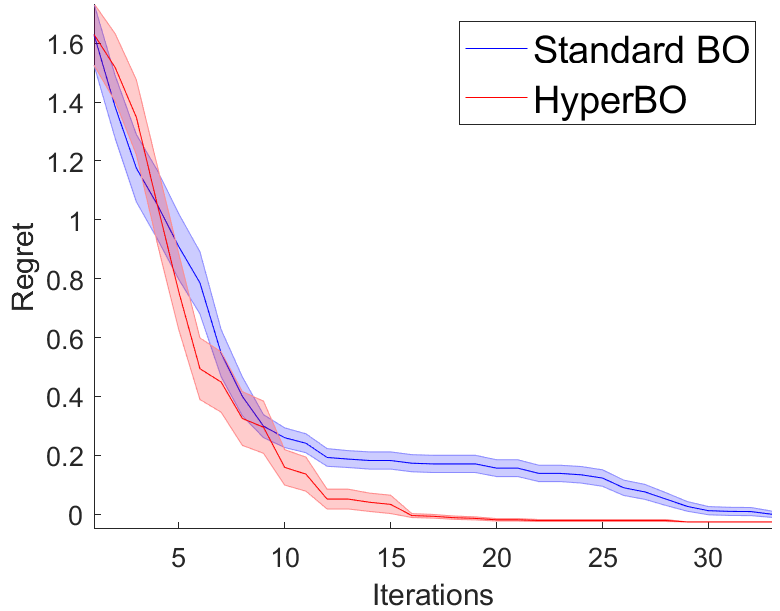}}}%
	\qquad
	\subfloat[Boston]{{\includegraphics[height = 3cm,width=3.5cm]{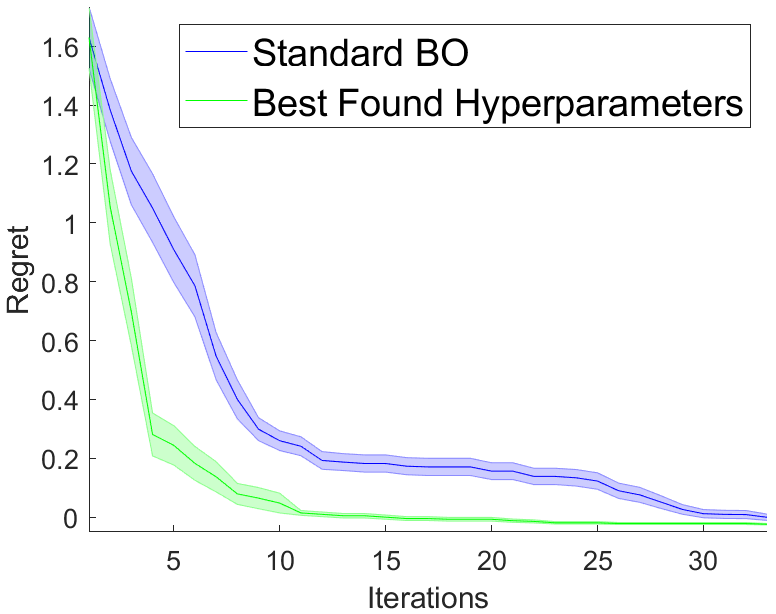}}}%
	\caption{Regret vs Iteration results of Monotonicity tuning experiments Case 1: a,c,e) Standard BO outperforms HyperBO in early iterations but converges faster in later iterations as monotonicity information is extracted. b,d,f) Case 2: GP with best monotonicity information outperforms Standard BO and is close in performance to GP with prior monotonicity information (in Goldstein-Price function) (gold standard).}%
	\label{fig:Monotonicity_experiments}%
\end{figure}

\subsection{Monotonicity Experiments}
For the monotonicity experiments, the goal is to find the optima quicker and also reveal the existence of monotonicity between variables and the output. The values of $\boldsymbol{\theta}$ were discretized to [-6, 0]. In addition, a constraint was applied such that a value of [-6] could not be applied in both the negative and positive directions of a given dimension. This reduces the search space with the assumption that a given dimension could not be both monotonically increasing and decreasing at the same time. For the following experiments, the value of $\lambda$ in the scoring function (\ref{eq:score2}) was set to $\frac{1}{2 \times 6 \times d}$, to normalize the regularization term. Figure \ref{fig:Monotonicity_experiments}a),c),e) shows the performance of the Predictive Modeling framework compared to a Standard BO. Figure \ref{fig:Monotonicity_experiments}b),d),f) compares the performance of the best discovered hyperparameters compared to the Standard BO, and where possible, against prior knowledge of monotonicity (synthetic function)(the gold standard).

\subsubsection{Virtual Derivatives}
The number of virtual derivatives was set to $5d$ where $d$ is the dimension of the data. This was to allow the virtual derivatives to enforce monotonicity strongly in the early iterations (when observations are low), and for observations to control the fit of the GP in later iterations (when observations are higher than virtual derivatives). These points were randomly placed in the search space.

\subsubsection{Goldstein-Price Function}
The trend towards the maxima of the Goldstein-Price function can be seen to decrease with respect to the first dimension, and increasing with respect to the second dimension (gold standard). When applying HyperBO, the average of the best monotonicity information from the trials returned a value of -3.24 (negative trend) for $x1$ and 2.64 (positive trend) $x2$. The results reflect the expected monotonicity, demonstrating how HyperBO is able to extract monotonicity information for the latent function during the optimization process.

\subsubsection{Diabetes}
The Diabetes dataset comprises of 10 variables (Age, Sex, BMI, Blood Pressure and 6 blood serum variables), from 443 patients, to predict the diabetes progression 1 year after baseline, and the task is to determine the patient with the highest disease progression.  The data of a particular patient was found to be an outlier and removed (the individual had a high disease progression, despite being quite young). Figure \ref{fig:Diabetes_mono}a) displays the correlation between the input variables and the output. This is used as a metric to assess our findings of monotonicity against. Figure \ref{fig:Diabetes_mono}b) displays the average direction of monotonicity for each feature.  For all but BP, S1, S2 and S5, monotonicity matched the correlation coefficient direction.

\begin{figure}
	\centering
	\subfloat[Correlation values]{{\includegraphics[height = 2.5cm, width=2.6cm]{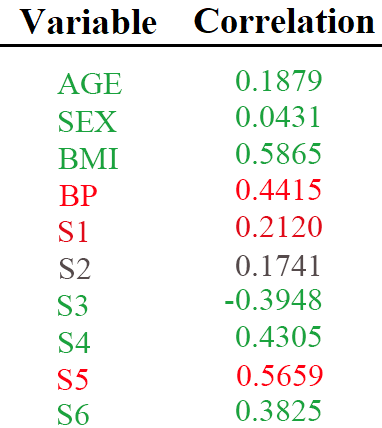} }}%
	\qquad
	\subfloat[Difference of average monotonicity value]{{\includegraphics[height = 2.5cm, width=3.5cm]{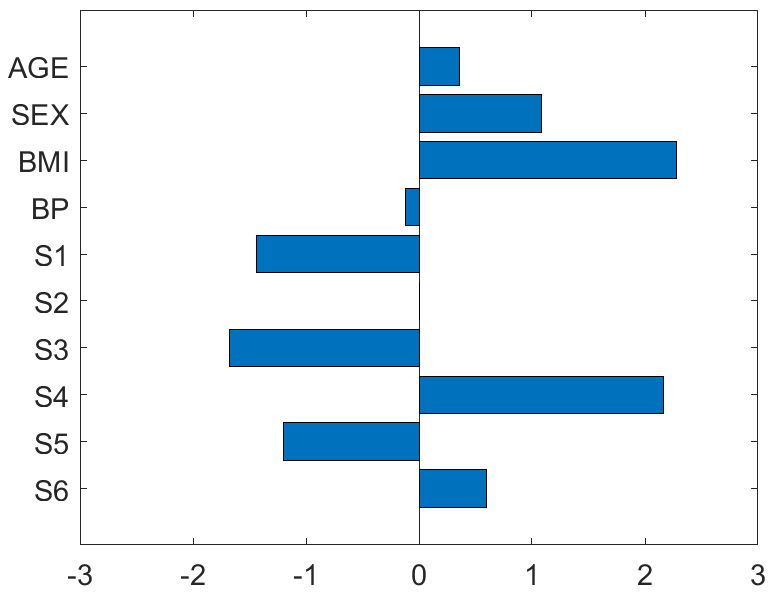} }}%
	\caption{Diabetes Dataset results. Green - Monotonicity matches correlation. Red - mismatch.}%
	\label{fig:Diabetes_mono}%
\end{figure}

\subsubsection{Boston Housing}
The Boston Housing dataset comprises of 13 features predicting the  median house price. The data was scaled with the median house price ranging between -2.5 and 2.5. For our experiment, only per capita crime rate (CRIM), non-retail business portion (INDUS), average number of rooms (RM), weighted distance to employment centres (DIS), pupil-teacher ratio (PTRATIO) portion of black residents (B), and \% of lower status (LSTAT)  were considered. The task was to discover the maximum median house price, whilst applying monotonic trends within the dataset. As there are multiple maxima in the dataset, the data was filtered to include only 1 maxima which was used in all experiments.

Figure \ref{fig:Boston_mono}a) shows the correlation coefficients between each input feature and the output. Figure \ref{fig:Boston_mono}b) displays the difference in the average monotonicity values for each dimension.

\begin{figure}
	\centering
	\subfloat[Correlation values]{{\includegraphics[height = 2.1cm, width=2.6cm]{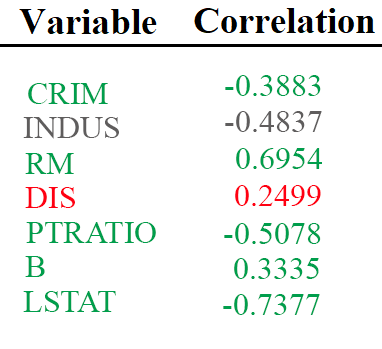} }}%
	\qquad
	\subfloat[Difference of average monotonicity value]{{\includegraphics[height = 2.1cm, width=3.5cm]{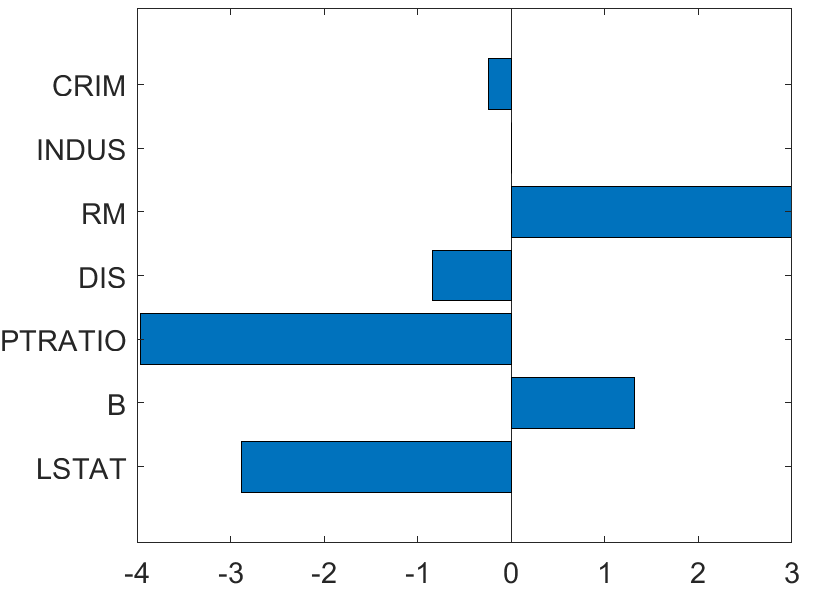} }}%
	\caption{Boston Housing dataset results. Green - Monotonicity matches correlation. Red - mismatch.}%
	\label{fig:Boston_mono}%
\end{figure}

\iffalse
Having observed the performance of the Predictive Monotonicity model search framework with HyperBO the following trends have emerged:
\begin{itemize}
	\item HyperBO performs best for complex functions which require a large number of samples. This is to balance the need for the number of models of monotonicity HyperBO must test. The number of monotonic models tested is equal to $(\frac{1}{K})$ of the number of total iterations $R$. A balance must be made in the selection of $K$ as, though a small value will result in more monotonic models being tested, it makes the scoring of these models unstable, as it is less lenient on the performance of the model and may score its performance poorly based on a few iterations governed by $K$.
	\item The framework was able to successfully extract most trends in the synthetic functions  as well as the real world datasets. It was interesting to observe in the results of the Diabetes dataset that as the maximum value did not follow on the trend of the data points, our framework monotonicity values were different to the correlation coefficient of the data.
	\item Exploring only monotonicity as derivative related information to be discovered and incorporated into a GP model can be limiting. Functions are often complex and not completely monotonous across an entire dimension range and as such further work is encouraged in developing this limitation further.  
\end{itemize}
\fi

\section{Conclusion}

This paper presented a novel mean of model selection during Bayesian Optimization search. The Predictive Modeling framework with HyperBO allows searching amongst the model space of GP's and assesses the performance of a given model in its progression toward the function optima. A scoring function was developed to assess the performance of a selected model at any given iteration. Experimental results applying the framework for monotonicity and length scale tuning have demonstrated its effectiveness in outperforming standard BO. Theoretical analysis proves convergence. \iffalse In addition applying model selection via HyperBO allows information about the black box function to be extracted during the optimization process. \fi

\bibliographystyle{named}
\bibliography{HyperBO_references}

\end{document}